\documentclass[a4paper,11pt,twoside,notitlepage]{amsart}

\usepackage[utf8]{inputenc}
\usepackage[T1]{fontenc}
\usepackage[margin=1in]{geometry}
\usepackage{amsmath,amssymb,amsthm,mathtools}
\usepackage{graphicx}
\usepackage{booktabs}
\usepackage{array}
\usepackage{tabularx}
\usepackage[shortlabels]{enumitem}
\usepackage{float}
\usepackage{xcolor}
\usepackage[hidelinks]{hyperref}
\usepackage{url}
\usepackage{microtype}

\numberwithin{equation}{section}
\graphicspath{{figures/}}
\allowdisplaybreaks

\newtheorem{theorem}{Theorem}[section]
\newtheorem{proposition}[theorem]{Proposition}
\theoremstyle{definition}
\newtheorem{definition}[theorem]{Definition}
\theoremstyle{remark}
\newtheorem{remark}[theorem]{Remark}

\newcommand{\dd}[2]{\frac{d#1}{d#2}}

\newcommand{\Vtrain}{T_{\mathrm{train}}}
\newcommand{\Veval}{T_{\mathrm{eval}}}
\newcommand{\Ef}{E_f^{\mathrm{full}}}
\newcommand{\Rmse}{\mathrm{RMSE}}
\newcommand{\MSE}{\mathrm{MSE}}
\DeclareMathOperator*{\argmin}{arg\,min}

\makeatletter
\@namedef{subjclassname@2020}{%
  \textup{2020} Mathematics Subject Classification}
\makeatother

\title[AI-Assisted Tumor-Growth Law Discovery]{AI-Assisted Tumor-Growth Law Discovery from Sparse Data: Refitting and LLM-Based Ranking}

\author[Y. Lu]{Yitian Lu}
\author[W. Wang]{Weinan Wang}
\thanks{Department of Mathematics, University of Oklahoma, Norman, OK, USA.
E-mail addresses: \texttt{yitian.lu-1@ou.edu} (Y. Lu),
\texttt{ww@ou.edu} (W. Wang).}

\date{}


\hypersetup{
    pdftitle={AI-Assisted Tumor-Growth Law Discovery from Sparse Data: Refitting and LLM-Based Ranking},
    pdfauthor={Yitian Lu and Weinan Wang},
    pdfsubject={AI-assisted inverse modeling and law discovery for sparse tumor-growth data},
    pdfkeywords={tumor growth, inverse problems, symbolic regression, PINNs, ODE-constrained optimization, large language models}
}

\begin{document}

\begin{abstract}
Inferring tumor-growth laws from sparse, noisy observations requires more than
accurate curve fitting. We study an AI-assisted workflow involving symbolic
regression, mathematical admissibility checks, coefficient refitting, and large
language model (LLM) ranking. Synthetic Gompertz data allow trajectory fit,
forecasting, and recovery of the governing growth law to be assessed separately.
A dense noiseless control recovers the generating family, but sparse noisy
sampling frequently leads to alternative forms. Among retained candidates in
the tested configurations, physics-informed neural network refitting often
increases growth-law error, whereas classical nonlinear least squares usually
reduces it. A separate exploratory analysis of archived LLM rankings shows no
consistent improvement over selection by training error and identifies errors
in mathematical reasoning. Plausible saturation, accurate trajectory fits, and
convincing explanations can therefore coexist with incorrect growth laws.
Mathematical validity and law recovery require evaluation alongside predictive
performance.
\end{abstract}

\maketitle
\enlargethispage{2pt}

\section{Introduction}
Serial measurements of tumor volume are among the simplest longitudinal data
available in experimental oncology, but even these data do not determine a
unique growth mechanism. Low-dimensional ordinary differential equations reduce
the spatial, cellular, vascular, and microenvironmental complexity of a tumor to
a state variable \(V(t)>0\) and a net growth law
\(\dot V=f(V)\). Such models are deliberately phenomenological: their value is
that they provide interpretable descriptions of growth, testable saturation
behavior, and extrapolations beyond the observation window. Gompertzian tumor
growth has a long history \cite{laird1964}, while logistic, Richards,
von Bertalanffy, power-law, and related families encode different assumptions
about early proliferation and the onset of growth limitation
\cite{benzekry2014,gerlee2013,sarapata2014}. Comparative studies show that a
family that describes one data set well need not be uniformly preferred across
tumor systems or forecasting horizons \cite{benzekry2014,murphy2016}.

This distinction separates two inverse problems. In a fixed-law problem, the
modeler selects a family and estimates its constants from measured volumes. In a
law-discovery problem, the functional form of \(f\) is itself unknown. The latter
contains the former but is more weakly determined: a short noisy trajectory may
constrain the visited portion of state space while leaving low-volume behavior,
the carrying capacity, and the approach to equilibrium largely unobserved.
Practical identifiability can therefore fail even within a prespecified sigmoid
family \cite{raue2009,simpson2022}; allowing the equation family to vary only
widens the ambiguity. Distinct laws may interpolate the same observations and
yet imply materially different doubling times, saturation levels, forecasts, or
treatment conclusions \cite{murphy2016}. A small training error is consequently
evidence for local trajectory agreement, not by itself evidence that the
governing rate law has been recovered.

Data-driven equation discovery offers a way to search beyond a short list of
named models. Free-form symbolic regression seeks compact analytical expressions
directly from measurements \cite{schmidt2009}; sparse identification selects a
small set of active terms from a prescribed feature library
\cite{brunton2016}; and modern systems such as PySR combine evolutionary search,
constant optimization, and an explicit accuracy--complexity tradeoff
\cite{cranmer2023}. For an autonomous scalar ODE, a natural construction is to
fit a smooth trajectory and regress estimated phase-space pairs
\((V,\dot V)\). Numerical
differentiation of noisy measurements is ill-posed, and smoothing choices trade
variance against bias \cite{vanbreugel2020}. Under sparse sampling, the symbolic
search may therefore fit artifacts of the derivative estimate or select a law
that behaves plausibly only over the narrow observed volume range. A small
phase-space loss does not establish mathematical admissibility or accuracy of
the candidate's ODE trajectory.

Refitting the constants of a frozen symbolic structure provides a second test.
ODE-constrained nonlinear least squares compares the trajectory generated by the
candidate with the observed volumes. Physics-informed neural networks (PINNs)
provide a complementary route by fitting a neural trajectory while penalizing
the candidate ODE residual \cite{raissi2019}. Their flexibility does not remove
the inverse-problem difficulties: composite objectives can have imbalanced
gradients, difficult loss landscapes, and optimization-dependent failure modes
\cite{krishnapriyan2021,wang2021}. Symbolic candidates add a further issue that
is muted in standard PINN demonstrations: logarithms, fractional powers, and
parameter-dependent singularities can make a candidate real-valued only on a
restricted state domain. Meaningful comparison therefore requires
candidate-aware parameterizations and explicit statuses for invalid domains,
solver failures, stagnation, and non-convergence.

Large language models introduce a different kind of scientific prior. Recent
systems use LLMs or foundation models to propose equation structures, construct
feature libraries, or guide iterative symbolic search
\cite{mower2025,shojaee2024}; recent benchmarks also emphasize the need to
separate genuine equation discovery from recognition or memorization of familiar
laws \cite{shojaee2025bench}. Here the LLM ranks accepted candidate equations;
it does not generate equations, refit constants, or
decide mathematical admissibility. It receives a table of candidate formulas,
complexities, and training errors, together with the observations and a statement
of the filters already passed. Research on LLM evaluators suggests that comparative judgments can be
useful but can also depend on presentation order and other prompt features
\cite{wang2024,zheng2023}. We examine six archived rankings under one common
prompt, compare their top choices with a training-error baseline, and assess
the mathematical content of their explanations. The records permit a
descriptive assessment of these choices, not an estimate of performance across
models or presentation orders.

The benchmark addresses three linked questions:
\begin{enumerate}[label=(Q\arabic*),leftmargin=3em,itemsep=2pt]
    \item Can symbolic regression recover a defensible growth-law family as
    measurements become sparse and noisy?
    \item Do classical and PINN refitting improve trajectory prediction without
    necessarily improving recovery of the vector field?
    \item Do the recorded LLM choices, made from formulas and training information,
    yield competitive hidden forecast and law-level performance?
\end{enumerate}
These questions are examined in a controlled synthetic setting, where
``discovery'' means recovery of a known generating law from hidden synthetic
truth. The controlled setting makes it possible to evaluate three objectives
that would not all be observable in clinical data:
\[
\text{fit the observed points}, \qquad
\text{forecast the trajectory}, \qquad
\text{recover the vector field } f(V).
\]
The archived LLM input files omit the generating equation, held-out volumes,
forecast errors, and vector-field errors. The saved rankings are scored against
these quantities without changing any ranked expression. The original pool
construction and its consequences for blinding are distinguished from the
training-only numerical protocol below.

\begin{enumerate}[leftmargin=2em,itemsep=2pt]
    \item An AI-assisted benchmark combines symbolic candidate generation,
    deterministic admissibility filtering, and classical and PINN-style
    constant refitting with a separate analysis of archived LLM rankings.
    \item The evaluation separates training fit, held-out forecasting, and vector-field
    recovery and detects plausible saturation
    surrogates---including \(\log(C-V)\)-type laws---that fit observed data while
    retaining the wrong low-volume or law-level structure.
    \item Candidate-aware domain handling and explicit numerical statuses record
    unsupported domains, non-convergence, NaN losses, and unstable refits as
    controlled outcomes. For supported parameterized domains, smooth state maps
    are defined to keep PINN iterates within the candidate domain.
    \item An exploratory LLM case study evaluates six complete rankings of 145
    supplied candidates. Hidden-error comparisons with a deterministic baseline
    and an audit of the explanations separate ranking outcomes from claims of
    scientific understanding or mathematical certification.
\end{enumerate}

\section{Tumor-growth laws and benchmark data}
The benchmark considers autonomous scalar laws
\begin{equation}
    \dd{V}{t}=f(V;\theta), \qquad V(0)=V_0>0.
\end{equation}
The main positive-control benchmark is Gompertz growth
\cite{benzekry2014,gompertz1825,laird1964},
\begin{equation}
    f_{\mathrm{G}}(V)=aV\log(K/V),
    \label{eq:gompertz}
\end{equation}
with analytic solution
\begin{equation}
    V(t)=K\exp\left(\log(V_0/K)e^{-at}\right).
\end{equation}
Unless otherwise stated, the synthetic Gompertz benchmark uses
\begin{equation}
    V_0=0.1, \qquad a=0.5, \qquad K=20.
\end{equation}
Time and volume are expressed in consistent arbitrary units.
For context, common reference families include logistic,
Richards/generalized logistic \cite{richards1959}, and von Bertalanffy-type
\cite{vonbertalanffy1957} laws:
\begin{align}
    f_{\mathrm{L}}(V)&=rV(1-V/K), \\
    f_{\mathrm{R}}(V)&=rV\left(1-(V/K)^\nu\right), \\
    f_{\mathrm{B}}(V)&=\alpha V^p-\beta V, \qquad 0<p<1.
\end{align}
Classical tumor-growth studies compare these and related laws, and no single
law is uniformly preferred across all data sets
\cite{benzekry2014,murphy2016}.
A separate illustration uses a Richards generating law to examine fitting with
a misspecified Gompertz family.

Synthetic observations are generated by solving the ODE and sampling
\begin{equation}
    y_i = V(t_i)(1+\sigma\epsilon_i),
    \qquad
    \epsilon_i\sim \mathcal N(0,1)\ \text{conditioned on}\
    \epsilon_i\ge -1/\sigma,
    \quad \sigma>0,
\end{equation}
which guarantees nonnegative multiplicative observations; at \(\sigma=0\),
\(y_i=V(t_i)\). The principal sparse configuration uses
\begin{equation}
    (t_i)_{i=1}^{6}=(0,3,6,9,12,15),
\end{equation}
and evaluates forecasts at six noiseless times from \(15.5\) through \(25.5\).
Other noise levels and sample sizes are identified in their corresponding
tables or captions, including the dense positive control with \(N=50\). The
ODE solvers and filters start from the known value \(V_0=0.1\). In the
simulations, multiplicative noise is also applied at \(t=0\), while the fitted
ODEs impose \(V(0)=V_0\). Thus the residual at the first observation includes
measurement noise even for the exact generating law. Training errors are
computed against all supplied observations, including this initial point.
Held-out volumes and synthetic-oracle error columns are excluded from the
archived LLM input files.

\begin{figure}[htbp]
\centering
\includegraphics[width=0.65\textwidth]{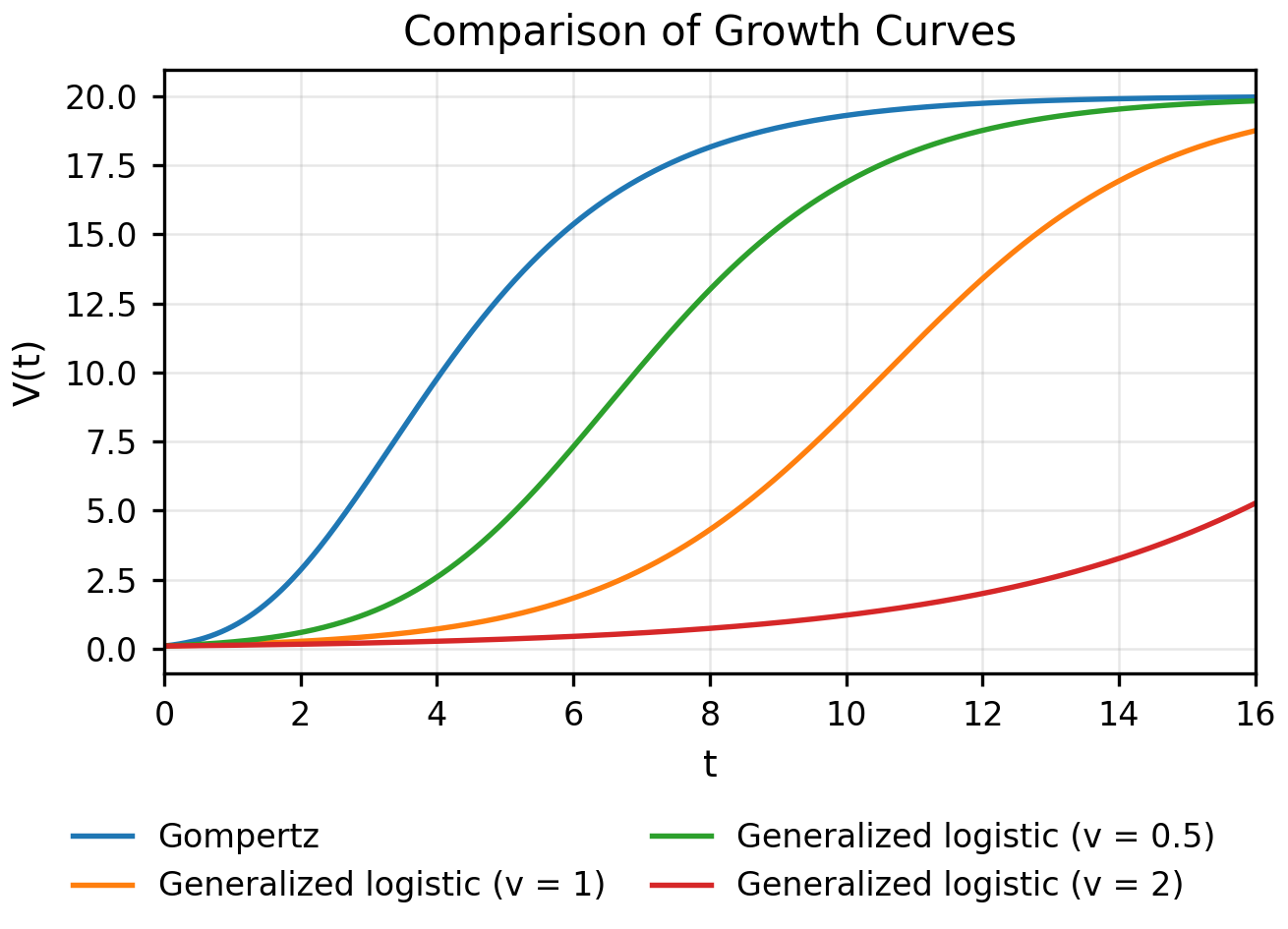}
\caption{Representative trajectories with \(K=20\) and \(V_0=0.1\).
Gompertz uses \(a=0.5\); the generalized-logistic curves use
\(\nu=0.5,1,2\) and \(r=0.5/\nu\) in the Richards convention above.
Sparse uniform sampling can concentrate observations near saturation, where
several laws are difficult to distinguish.}
\label{fig:laws}
\end{figure}

\subsection{Near-zero behavior of the Gompertz vector field}
The logarithm in \eqref{eq:gompertz} diverges as \(V\to0^+\), but the vector
field itself has a removable zero-limit:
\begin{equation}
    \lim_{V\to0^+} V\log(K/V)=0.
\end{equation}
The derivative with respect to \(V\) is singular,
\begin{equation}
    \frac{d}{dV}\left[aV\log(K/V)\right]
    = a\left(\log(K/V)-1\right),
\end{equation}
which matters numerically. In PINN training, the immediate risk is that the
neural surrogate evaluates the logarithm at nonpositive values, or that a
candidate law has a restricted domain such as \(V<C\).

\section{The AI-assisted discovery and refitting pipeline}
\label{sec:pipeline}
The pipeline has five modules:
\[
\begin{gathered}
\text{data}
\longrightarrow
\text{candidate generation}
\longrightarrow
\text{admissibility filtering}\\
\longrightarrow
\text{refitting}
\longrightarrow
\text{ranking and evaluation}.
\end{gathered}
\]
LLM ranking forms the final stage; the archived ranking case study is described
in Section~\ref{sec:llm-evaluation}.

\subsection{Candidate generation by symbolic regression}
The discovery route first fits a smooth surrogate \(\widetilde V(t)\). It then
differentiates this surrogate to obtain \(\widetilde{\dot V}(t)\) and feeds
the phase-space pairs
\begin{equation}
    \left(\widetilde V(s_q),\widetilde{\dot V}(s_q)\right),
    \qquad q=1,\ldots,200,
\end{equation}
into symbolic regression. In sparse noisy data this step is fragile because
numerical differentiation amplifies noise and sparse uniform samples may contain
little information about the early-growth part of the vector field
\cite{vanbreugel2020}.

Unless stated otherwise, the smoothing spline has degree
\(k=\min(3,N-1)\) and smoothing factor
\begin{equation}
    S=N\left[0.05\left(\max_i y_i-\min_i y_i\right)\right]^2.
\end{equation}
It is evaluated on 200 equally spaced times over the training interval. The
symbolic-regression runs use 1000 iterations, squared loss, a maximum expression
size of 20, addition, subtraction, multiplication, division, exponentiation,
the unary operator \(\log\), deterministic serial execution, and random seed 77. Every expression
in the resulting hall of fame is passed to the subsequent filter; phase-space
loss alone does not determine the final candidate.
Operator complexity weights are 1 for addition and subtraction, 2 for
multiplication and division, and 3 for exponentiation and logarithms. The dense
noiseless control uses an interpolating quintic spline (\(S=0\)); the
30-realization sparse-search comparison also uses degree five with the
smoothing factor above. The sampling-density diagnostic identifies its cubic
and quintic settings separately.

\subsection{Rule-based admissibility filters}
A candidate \(f_m\) is screened before refitting. The filter returns both an
accept/reject decision and a primary status.

\begin{definition}[Candidate state domain]
For a symbolic candidate \(f_m(V;\theta_m)\), its admissible state domain is
\begin{equation}
    D_m(\theta_m)
    =
    \left\{
        V>0:
        f_m(V;\theta_m)\in\mathbb R
        \text{ and }
        \lvert f_m(V;\theta_m)\rvert<\infty
    \right\}.
    \label{eq:candidate-domain}
\end{equation}
A trajectory is domain-valid on a time interval \(I\) if
\(V(t)\in D_m(\theta_m)\) for every \(t\in I\).
\end{definition}

The main rejection categories are
\begin{equation}
\begin{gathered}
\text{domain failure},\quad
\text{no positive equilibrium},\quad
\text{unstable equilibrium},\\
\text{negative trajectory},\quad
\text{unsolvable ODE}.
\end{gathered}
\end{equation}
PySR caps expression complexity during generation; complexity is not a
post-generation filter status.

The filter checks the following conditions.
\begin{enumerate}[leftmargin=2em,itemsep=2pt]
    \item \textbf{Autonomy.} The candidate is a function of the single PySR
    predictor \(V\), not explicit time.
    \item \textbf{Domain validity.} The expression must be real-valued and finite
    on the positive-volume interval required for numerical checks.
    \item \textbf{Stable equilibrium.} When saturation is expected, there must
    exist \(V^\star>0\) such that
    \begin{equation}
        f_m(V^\star)=0, \qquad f_m'(V^\star)<0.
    \end{equation}
    The equilibrium need not lie inside the observed training-volume range.
    The initial condition must belong to its basin of attraction; a stable
    root separated from \(V_0\) by another equilibrium is insufficient.
    \item \textbf{Forward positivity.} The ODE solution from \(V_0\) must remain
    positive on the training and forecast interval:
    \begin{equation}
        \min_{0\le t\le \Veval} V_m(t;\theta)>0.
    \end{equation}
    \item \textbf{Solver status.} NaN evaluation and failure to obtain a finite
    forward solution are recorded as candidate-level statuses.
\end{enumerate}

For the principal numerical refitting comparisons, define
\[
 L=\min\{V_0,\min_{i:y_i>0}y_i\},\qquad
 U=\max\{V_0,\max_i y_i\}.
\]
The expression must be real and finite on \([L,U]\). Equilibria are searched
on \([L,2U]\), and the domain check also covers the interval from \(V_0\) to
the first equilibrium in its direction of flow. That equilibrium must be
stable. Symbolic real-domain and root checks supplement numerical screening;
an unresolved required check is not treated as a pass. The forward solution is
checked through the fixed forecast horizon. Neither future volumes nor the
generating equation enters these acceptance decisions. The same checks are
applied to the original expression and to each restored refit.

The forward-positivity condition is essential. A candidate may possess a stable
positive equilibrium and still be invalid from the observed initial condition.
For example,
\begin{equation}
    f(V)=-V\bigl(\log(V+5.912802)-3.3113704\bigr)-1.5374633
    \label{eq:nonphysical}
\end{equation}
has \(f(0)<0\), so small initial volumes are driven downward. It has a stable
positive equilibrium near saturation, but the initial condition \(V_0=0.1\) lies
on a forward trajectory that becomes biologically inadmissible.

\begin{figure}[htbp]
\centering
\includegraphics[width=0.66\textwidth]{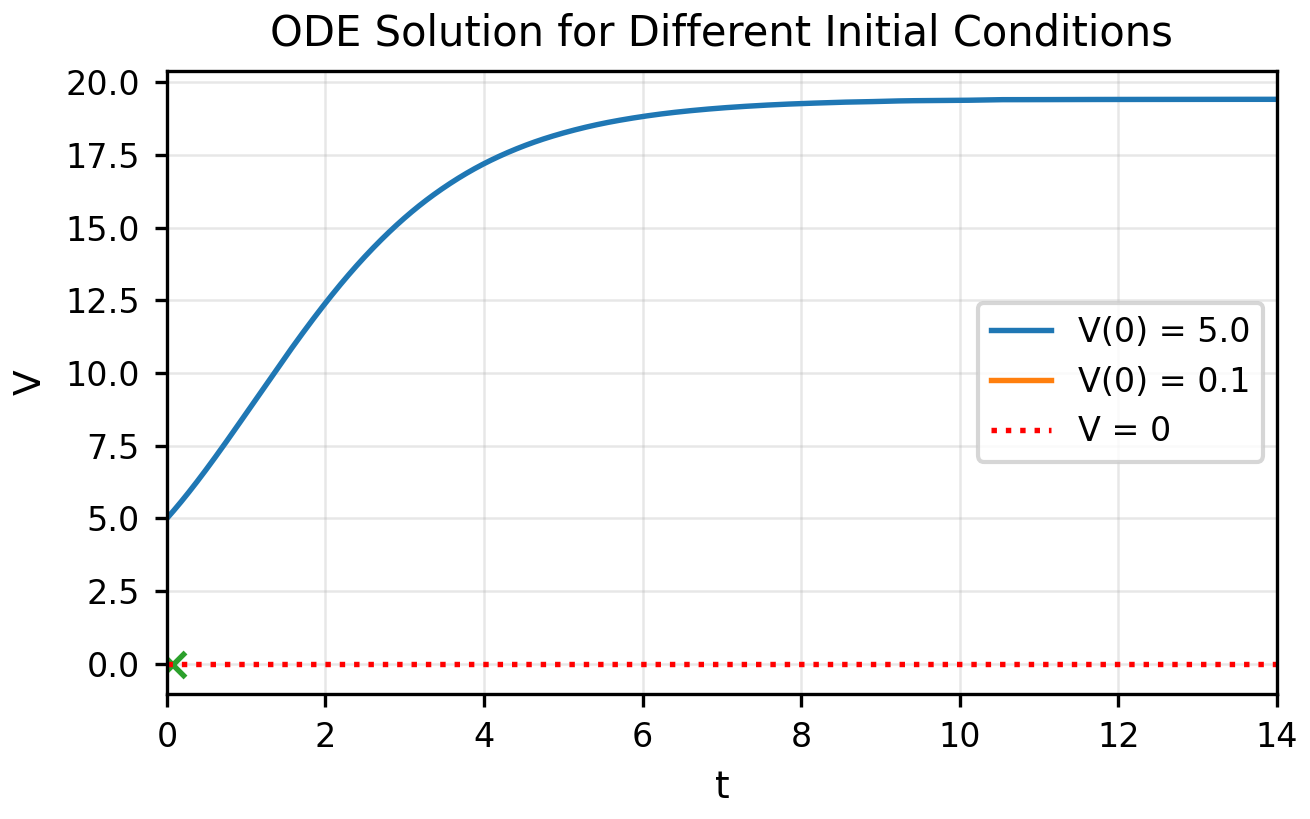}
\caption{A candidate law can have a stable positive equilibrium and still be
inadmissible from the observed initial condition. The forward-positivity filter
rejects candidates that drive \(V(t)\) negative or leave the candidate's domain
on the evaluation interval.}
\label{fig:nonphysical}
\end{figure}

\subsection{Classical ODE-constrained nonlinear least squares}
For each accepted candidate \(f_m(V;\theta_m)\), constants are refit by solving
\begin{equation}
    \dd{V}{t}=f_m(V;\theta_m), \qquad V(0)=V_0,
\end{equation}
and minimizing trajectory residuals against the observed data. Writing
\begin{equation}
    r_m(\theta_m)
    =
    \bigl(
        V_m(t_1;\theta_m)-y_1,\ldots,
        V_m(t_N;\theta_m)-y_N
    \bigr)^\top,
\end{equation}
the objective is
\begin{equation}
    J_m(\theta_m)
    =
    \frac12\|r_m(\theta_m)\|_2^2,
    \qquad
    \widehat\theta_m^{\mathrm{NLS}}
    =
    \argmin_{\theta_m\in\mathbb R^{p_m}}J_m(\theta_m).
    \label{eq:nls}
\end{equation}
The implementation uses unconstrained Levenberg--Marquardt nonlinear least
squares with at most 2000 residual evaluations
\cite{levenberg1944,marquardt1963}. Each residual evaluation solves the ODE from
the known initial volume. If integration fails or returns a nonfinite
trajectory, the residual function returns a large finite vector so that invalid
values do not propagate through the optimizer. The restored expression is then
re-evaluated by the same admissibility checks used for the original candidate.

The symbolic structure is frozen during this refit. Constants may move, but
exponents remain fixed unless explicitly included as model parameters.
For example,
\begin{equation}
    V^2(a-bV)+c
\end{equation}
may become
\begin{equation}
    V^2(\widehat a-\widehat bV)+\widehat c,
\end{equation}
but not
\begin{equation}
    V^{0.41}(\widehat a-\widehat bV)+\widehat c.
\end{equation}
The latter is a different law family.

Optimizer termination and admissibility of the restored equation are distinct
checks. A finite trajectory does not establish convergence of the optimizer,
and unchanged constants alone do not distinguish a stationary point from
stagnation. Accordingly, parameter movement is interpreted together with the
optimizer's termination status; no first-order optimality claim follows solely
from a small change in constants.

\subsection{Candidate-aware PINN refitting}
The PINN arm fits a neural surrogate \(V_\phi(t)\) and candidate constants
\(\theta_m\) by minimizing
\begin{equation}
    \mathcal L(\phi,\theta_m)
    =
    \mathcal L_{\mathrm{data}}
    +
    \lambda_{\mathrm{phys}}\mathcal L_{\mathrm{phys}},
\end{equation}
where
\begin{equation}
    \mathcal L_{\mathrm{data}}
    =
    \frac1N\sum_{i=1}^N
    \left(\frac{V_\phi(t_i)-y_i}{V_{\mathrm{scale}}}\right)^2
\end{equation}
and
\begin{equation}
    \mathcal L_{\mathrm{phys}}
    =
    \frac1Q\sum_{q=1}^Q
    \left(
    \frac{\frac{dV_\phi}{dt}(s_q)-f_m(V_\phi(s_q);\theta_m)}
    {V_{\mathrm{scale}}/T_{\mathrm{scale}}}
    \right)^2.
\end{equation}
The candidate right-hand side is evaluated only after the neural output has been
mapped into the admissible state domain \(D_m(\theta_m)\).

The candidate-aware PINN protocol specifies two hidden \(\tanh\) layers of width 20. Time
is normalized to \([0,1]\), and volume and residual scales use the observed
training-volume range. Unless a physics-weight comparison states otherwise,
\(\lambda_{\mathrm{phys}}=0.1\). Training uses 1000 collocation times, Adam with
learning rates \(10^{-2}\) for network weights and \(10^{-1}\) for normalized
candidate constants, which are mapped back to physical constants for the ODE,
300 epochs, and gradient clipping at norm 1. The principal sparse
protocol fixes the data-realization seed at 47 and the PINN initialization seeds
at 43--62, giving 20 starts per candidate. The observed initial point enters the
data loss; no separate hard initial-condition constraint is imposed.
A retained PINN start completes this fixed training budget with finite losses,
passes the restored-equation filter, and has finite solved-ODE training MSE.
This numerical-success definition does not assert first-order optimality.

Let
\begin{equation}
    \rho(z)=\frac{1}{1+e^{-z}}
\end{equation}
denote the logistic map. For candidates whose only state restriction is
positivity, the output map is
\begin{equation}
    V_\phi(t)
    =
    \varepsilon_0+V_{\mathrm{scale}}
    \operatorname{softplus}(z_\phi(t)).
    \label{eq:positive-map}
\end{equation}
where \(\varepsilon_0=10^{-6}V_{\mathrm{scale}}\).
For a supported bounded domain \(D_m(\theta_m)=(0,U_m(\theta_m))\), the output map is
\begin{equation}
    V_\phi(t)
    =
    \varepsilon_l+
    \bigl(U_m(\theta_m)-\varepsilon_u-\varepsilon_l\bigr)\rho(z_\phi(t)),
    \label{eq:bounded-map}
\end{equation}

\begin{proposition}[Domain preservation]
If \(\varepsilon_l,\varepsilon_u>0\) and
\(\varepsilon_l+\varepsilon_u<U_m(\theta_m)\), then the output map in
\eqref{eq:bounded-map} satisfies
\begin{equation}
    \varepsilon_l<V_\phi(t)<U_m(\theta_m)-\varepsilon_u
\end{equation}
for every \(t\) and every finite network output \(z_\phi(t)\).
\end{proposition}

\begin{proof}
Since \(0<\rho(z)<1\) for every \(z\in\mathbb R\), multiplication by
\(U_m(\theta_m)-\varepsilon_u-\varepsilon_l>0\) and addition of
\(\varepsilon_l\) give the
stated bounds.
\end{proof}

When the upper boundary depends on a refitted constant, domain preservation
requires evaluating \(U_m(\theta_m)\) at the current parameter values and
maintaining \(U_m(\theta_m)>\varepsilon_l+\varepsilon_u\).
The implementation uses \(\varepsilon_l=10^{-3}\) and
\(\varepsilon_u=10^{-6}\max\{1,|U_m(\theta_m)|\}\), with bounds evaluated
from the current physical constants. This margin prevents a floating-point
sigmoid value of one from reaching a strict logarithmic domain boundary.
The computational upper cap is \(10^6\); an interval of width at most one
or a failed bound evaluation receives a domain-failure status.
For an affine logarithmic constraint with lower boundary \(L_m(\theta_m)\),
the lower endpoint becomes
\(\max\{10^{-3},L_m(\theta_m)+10^{-6}\max(1,|L_m(\theta_m)|)\}\).
Multiple upper constraints are combined by their minimum and multiple lower
constraints by their maximum, evaluated at the current constants.

For symbolic candidates with disconnected or nested parameter-dependent domains
that cannot be parameterized reliably, the PINN attempt is assigned the status
\[
    \text{unsupported candidate domain}.
\]
The classical ODE-constrained arm is evaluated
independently and may still succeed.

\begin{remark}
If an invalid operation has already produced NaN inside the automatic
differentiation graph, replacing the final scalar loss by a fixed constant does
not restore a valid gradient; a literal constant also has zero derivative.
Likewise, perturbing a candidate parameter whenever an invalid evaluation occurs
changes the law during optimization, and clipping the network to the observed
volume range changes the state space on which the physics residual is enforced.
Accordingly, the candidate domain is enforced before \(f_m\) is evaluated; when
such enforcement is unavailable, the attempt receives a controlled domain
status.
\end{remark}

The network prediction \(V_\phi(t)\) and the ODE solution using the refitted
constants are reported separately. The principal forecast of the recovered law
is obtained by solving
\begin{equation}
    \dd{V}{t}=f_m(V;\widehat\theta_m^{\mathrm{PINN}}),
    \qquad V(0)=V_0.
\end{equation}
The neural-network forecast is a separate diagnostic. The discrepancy between
the two trajectories is measured by the dimensionless PINN--ODE consistency gap
\begin{equation}
G_{\mathrm{PINN-ODE}}
=
\frac{1}{V_{\mathrm{scale}}}\left(
\frac1J\sum_{j=1}^J
\left[V_\phi(t_j)-V_m(t_j;\widehat\theta_m^{\mathrm{PINN}})\right]^2
\right)^{1/2}.
\end{equation}
Valid consistency gaps evaluate both trajectories on a common grid in physical
time and volume units, then normalize by the fixed training-volume range.
The grid contains 1000 equally spaced training times. A large gap indicates that the neural surrogate and
restored ODE are inconsistent, but no thresholded pass/fail rule is imposed.

\subsection{Post-hoc LLM ranking}
The archived prompt defines a ranking and explanation task. Each input consists
of the noise level, training and forecast intervals, six training observations,
and a candidate table with four columns: candidate ID, fitted right-hand side,
symbolic complexity, and training-trajectory MSE. The prompt states that the
listed candidates have passed real-domain, positive-solution, and
stable-equilibrium checks. It withholds the generating equation and parameters,
future volumes, forecast errors, and law-level errors. The six input tables pool
original PySR, PINN-refitted, and classical-refitted expressions; method names
are omitted, but the rows are interleaved and the IDs sequential. Thus the
prompt's description of the table as randomized is not supported by the stored
inputs.

These historical candidate pools were formed before the training-only filter
and common-unit PINN--ODE gap described above. Their original filter used a
volume bound derived from the synthetic future trajectory. Accordingly, direct
withholding of the oracle columns does not establish oracle-independent pool
construction. The LLM analysis is a descriptive case study conditional on these
historical pools, not a blinded comparison of the corrected refitting arms.
Every saved ranking is evaluated on its exact supplied expressions, without
substituting a later refit or changing pool membership.

The prompt asks the evaluator to rank every candidate exactly once, applying
mathematical and biological robustness first, observed-data agreement second,
and parsimony and interpretability third. It cannot refit or simplify an
expression, create a new law, reject an accepted law, or override a filter. This
role separates the requested ranking from equation generation and numerical
optimization. The saved responses are assessed for both selection outcomes and
the validity of their explanations.

For synthetic data, LLM ranking performance is defined by
\begin{equation}
    \mathrm{LLMTop1}
    =
    \mathbf 1\{\text{LLM top-ranked candidate matches the generating family}\},
\end{equation}
forecast regret
\begin{equation}
    \mathrm{Regret}_{\mathrm{forecast}}
    =
    \Rmse_{\mathrm{forecast}}(\text{LLM top choice})
    -
    \min_m \Rmse_{\mathrm{forecast}}(m),
\end{equation}
and law-error regret
\begin{equation}
    \mathrm{Regret}_{f}
    =
    \Ef(\text{LLM top choice})-
    \min_m \Ef(m),
\end{equation}
where each minimum is taken over the complete supplied candidate pool for that
noise level. A regret is reported only when the corresponding
oracle metric is finite for every candidate in that trial; otherwise its
availability status is retained. \(\mathrm{LLMTop1}\) is defined only
when the generating family is represented in the supplied candidate pool;
otherwise the trial measures prioritization and regret but not top-one
family recovery.
The retrospective deterministic comparison chooses the smallest supplied
training-trajectory MSE, breaking exact ties by complexity and then candidate
ID. Its forecast and law errors are evaluated on the same equations and grids
as the saved LLM choices. This baseline is an author-computed comparison of the
archived outcomes.

\section{Evaluation metrics}
The candidate table contains both numerical metrics and status labels.

\paragraph{Training trajectory error.}
\begin{equation}
    \MSE_{\mathrm{train}}(m)=
    \frac1N\sum_{i=1}^N
    \left[V_m(t_i;\widehat\theta_m)-y_i\right]^2.
\end{equation}
For the PySR and refit-ODE rows, \(V_m\) is obtained by solving the candidate
ODE from \(V_0\). A direct PINN-network training MSE is reported separately and
is never substituted for the solved-ODE value.
The initial observation contributes a residual even for the generating law when
noise is applied at \(t=0\), as described in Section~2.

\paragraph{Forecast error.}
On a held-out time grid \(\{\tau_j\}_{j=1}^J\subset [\Vtrain,\Veval]\),
\begin{equation}
    \Rmse_{\mathrm{forecast}}(m)=
    \left(
    \frac1J\sum_{j=1}^J
    \left[V_m(\tau_j;\widehat\theta_m)-V_{\mathrm{true}}(\tau_j)\right]^2
    \right)^{1/2}.
\end{equation}
The principal forecast metric likewise uses the solved candidate ODE. Direct
network forecasts are diagnostics of the neural surrogate, not forecasts of the
restored governing law.

\paragraph{Law-level error.}
For synthetic Gompertz benchmarks, law recovery is measured by
\begin{equation}
\Ef(m)
=
\left(
\frac{
\int_{0.1}^{20}
\left[f_m(V;\widehat\theta_m)-0.5V\log(20/V)\right]^2\,dV
}{
\int_{0.1}^{20}
\left[0.5V\log(20/V)\right]^2\,dV
}
\right)^{1/2}.
\label{eq:ef}
\end{equation}
This dimensionless relative \(L^2\) vector-field error is a synthetic-oracle
metric for the restored candidate ODE; it is not defined for the direct network
output and is never supplied to the LLM.
If the candidate is not real-valued and finite on \([0.1,20]\), the metric is
recorded as a domain failure; the integration interval is not shortened.

\paragraph{Numerical statuses.}
Domain validity, optimizer termination, and integration outcomes accompany the
error metrics. A numerical comparison requires finite
values of the metric under consideration. Failure statuses include
\begin{equation}
\begin{gathered}
\text{Domain Failure},\quad
\text{Unsupported Candidate Domain or Expression},\quad
\text{NaN Loss},\\
\text{Unstable Refit},\quad
\text{Optimizer Stagnation},\quad
\text{Non-convergence}.
\end{gathered}
\end{equation}
The denominator and success criterion are stated with each summary. A finite
forecast alone does not establish biological admissibility or optimizer
convergence.

\subsection{Experimental units and paired summaries}
The independent experimental unit is a generated data realization. Candidate
expressions share the observations and derivative estimates from a
symbolic-regression run, and repeated PINN starts share a candidate structure.
Table~\ref{tab:ranking-instability} summarizes 30 data realizations. The remaining
numerical tables describe selected candidates or candidate sets within stated
realizations. Repeated starts are summarized within each candidate; the
representative reported by the main refitting program minimizes training MSE
among retained starts. For a metric \(M\), the
paired change for arm \(A\) is
\begin{equation}
    \Delta M^{A}_{r,m}
    =M\!\left(f^{A}_{r,m}\right)
    -M\!\left(f^{\mathrm{PySR}}_{r,m}\right),
\end{equation}
so a negative value denotes improvement over the original PySR constants.
Medians over candidate expressions describe that candidate set and are not
interpreted as independent treatment effects.
\section{Results}

\subsection{Dense noiseless positive control}
The dense noiseless example uses
\begin{equation}
    N=50, \qquad \sigma=0, \qquad 0\le t\le14.
\end{equation}
Its PINN refit uses 500 epochs, candidate-parameter learning rate
\(10^{-5}\), and initialization seed 42; other network settings agree with
the main configuration.
The recovered expression, rounded to the displayed precision,
\begin{equation}
    \dot V=V(1.50-0.500\log V)
    \label{eq:recovered-gompertz}
\end{equation}
is algebraically Gompertz. Indeed,
\begin{equation}
    aV\log(K/V)=aV\log K-aV\log V.
\end{equation}
Matching \eqref{eq:recovered-gompertz} gives
\begin{equation}
    a=0.500, \qquad a\log K=1.50,
\end{equation}
and therefore
\begin{equation}
    \log K=3,\qquad K=e^3\approx 20.09.
\end{equation}
The expression therefore belongs to the generating Gompertz family. The rounded
coefficients imply a carrying capacity close to, but not exactly equal to, the
generating value \(K=20\). This realization provides a positive control for
family recovery under dense noiseless sampling.

\begin{figure}[htbp]
\centering
\includegraphics[width=0.66\textwidth]{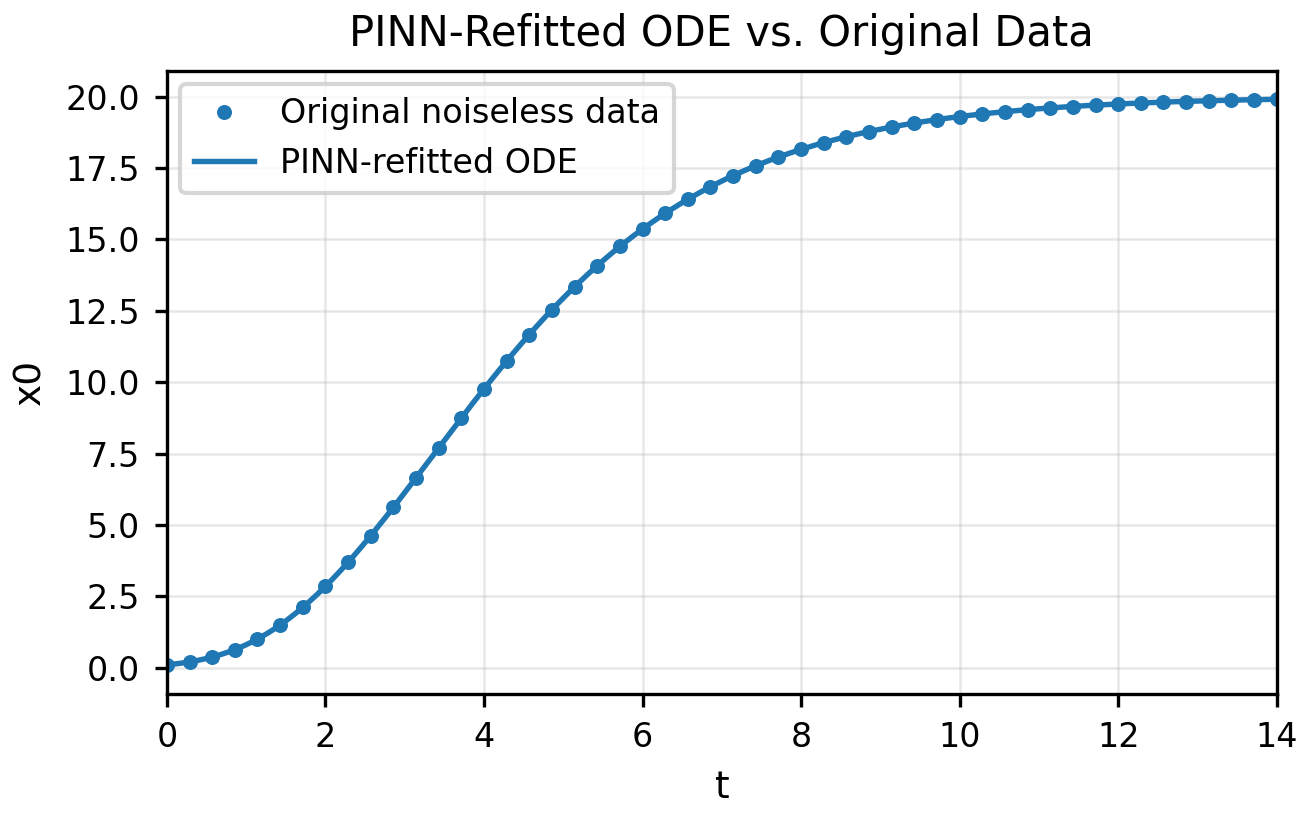}
\caption{Dense noiseless positive control with 50 observations on \([0,14]\).
The curve is the solved PINN-refitted ODE. The rounded recovered expression
\(V(1.50-0.500\log V)\) is algebraically equivalent to Gompertz with
\(a=0.5\) and \(K=e^3\approx20.09\).}
\label{fig:sanity}
\end{figure}

\subsection{Sparse symbolic search can favor plausible but incorrect laws}
Under sparse noisy sampling, the default symbolic-regression choice can have an
incorrect functional form. A representative surrogate, with \(C>1\), is
\begin{equation}
    f_C(V)=\log(C-V).
\end{equation}
This family has a stable equilibrium at
\begin{equation}
    V^\star=C-1,
\end{equation}
because \(\log(C-V^\star)=0\) and \(f_C'(V^\star)=-1<0\). Nevertheless it is not
Gompertz. It satisfies
\begin{equation}
    f_C(0)=\log C\ne0,
\end{equation}
whereas Gompertz satisfies
\begin{equation}
    \lim_{V\to0^+}0.5V\log(20/V)=0.
\end{equation}
Moreover, \(\log(C-V)\) is monotone decreasing in \(V\), while the Gompertz rate
is unimodal with maximum near \(V=K/e\). Hence a log-saturation surrogate can
track the saturation phase while missing the early and intermediate growth law.

A benchmark with Gompertz truth, \(N=6\), and \(8\%\) noise illustrates this
instability. For each of 30 data seeds (47--76), PySR generates a candidate set
from spline-derived phase-space data. Its default complexity-aware selection
returns a non-Gompertz expression in 29 cases. This selection balances
phase-space accuracy and complexity; it is not simply the expression with the
smallest phase-space error.

The generating equation at its known parameters has smaller trajectory RMSE
than the default PySR choice in 24 of 30 realizations, counting two failed
candidate integrations as infinite error. Among the 28 finite comparisons, the
count is 22. These errors are evaluated at the six noisy training observations.
Selecting the smallest training-trajectory RMSE from the PySR candidate set
returns a Gompertz expression in only 3 of 30 realizations. Thus neither the
default symbolic score nor training-trajectory error reliably identifies the
generating family in this sparse regime.

\begin{figure}[htbp]
\centering
\includegraphics[width=0.70\textwidth]{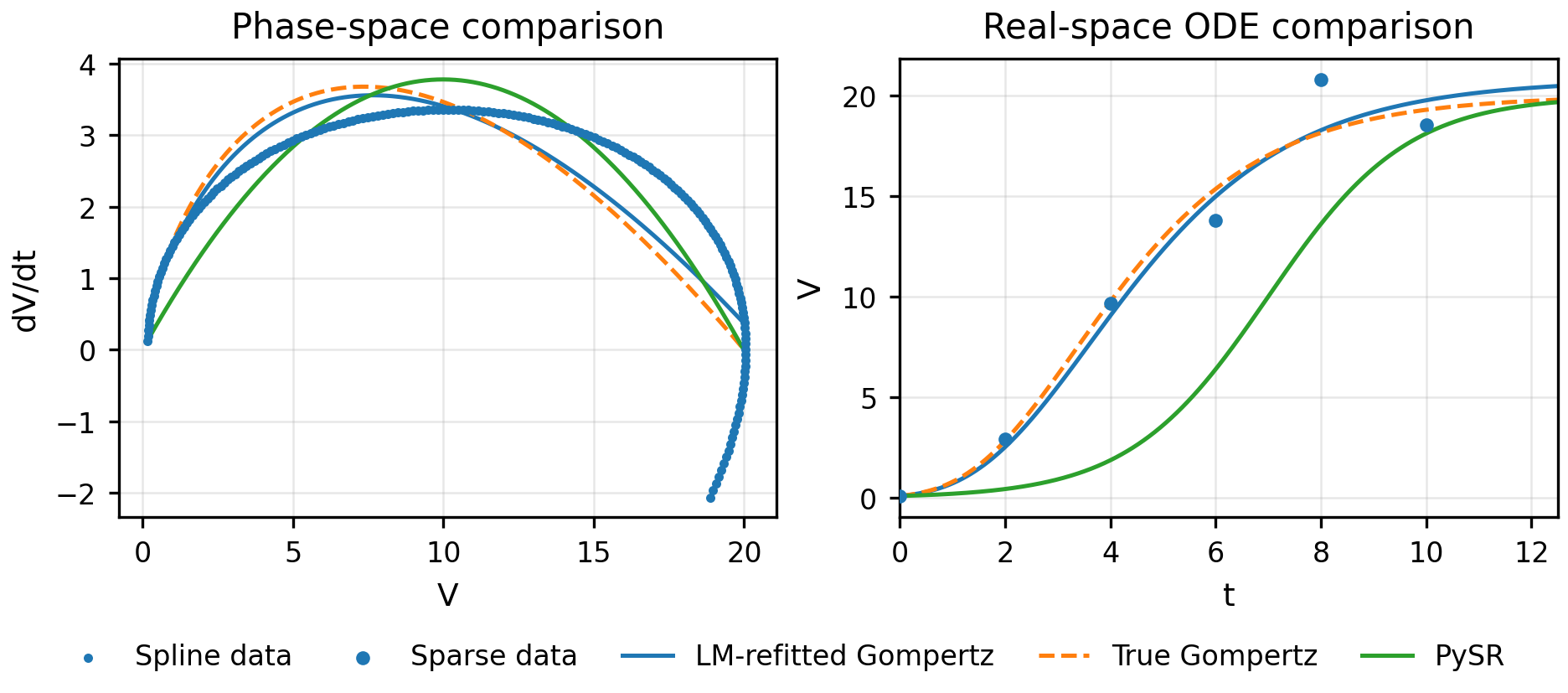}
\caption{A separate sparse example with \(N=6\), \(\sigma=0.08\), data seed 47,
and observations on \([0,10]\). The PySR law
\(V(0.7568282-0.037870117V)\) follows the estimated phase-space relation but
has a different trajectory from Gompertz truth. The blue Gompertz curve is
fitted to spline-derived \((V,\dot V)\) pairs, not to trajectory residuals.
Trajectories are displayed through \(t=12.5\).}
\label{fig:flip}
\end{figure}

\begin{table}[htbp]
\centering
\caption{Ranking instability across 30 sparse noisy Gompertz realizations.
Trajectory errors use the six noisy training observations; the generating-law
baseline uses the known parameters. Bracketed intervals are two-sided 95\%
Wilson intervals. The 24/30 comparison counts two failed candidate integrations
as infinite error; the corresponding finite-only count is 22/28.}
\label{tab:ranking-instability}
\begin{tabularx}{\textwidth}{@{}>{\raggedright\arraybackslash}X>{\raggedright\arraybackslash}p{0.34\textwidth}@{}}
\toprule
Quantity & Value \\
\midrule
Truth & Gompertz, \(a=0.5, K=20, V_0=0.1\) \\
Data regime & \(N=6,\ \sigma=0.08,\ R=30\) \\
Default PySR choice is non-Gompertz & \(29/30=97\%\) \([83\%,99\%]\) \\
Lowest training-RMSE candidate is Gompertz & \(3/30=10\%\) \([3\%,26\%]\) \\
Generating law has lower training RMSE than default choice & \(24/30=80\%\) \([63\%,90\%]\) \\
Median training RMSE, generating law & \(1.145\) \\
Median training RMSE, default PySR choice & \(1.815\) \\
\bottomrule
\end{tabularx}
\end{table}

\subsection{Refitting separates trajectory fit, forecast accuracy, and law recovery}
The seed-45 example in Table~\ref{tab:seed45} uses
\(N=6\), \(20\%\) noise, and seed 45. The candidate is a quadratic-saturation
law selected from the PySR hall of fame. PINN refitting increases both training
and forecast errors relative to the original constants. Classical nonlinear
least squares gives the smallest training and forecast errors, but the original
PySR expression has the smallest law-level error. Better trajectory prediction
therefore does not imply better recovery of the vector field.

\begin{table}[htbp]
\centering
\caption{Seed-45 summary for a structure-preserving quadratic candidate
\(f(V)=V(0.2526-0.01562V)+2.1931\) from a sparse noisy Gompertz benchmark
(seed 45, \(N=6\), \(\sigma=0.20\),
\(\lambda_{\mathrm{phys}}=0.1\)). The PINN representative uses initialization
seed 53, selected from 20 starts by training MSE. Smaller values are better. The column
\(\Ef\) denotes the full law-level error in \eqref{eq:ef}.}
\label{tab:seed45}
\begin{tabular}{lccc}
\toprule
Method & Training MSE & Forecast RMSE & \(\Ef\) \\
\midrule
PySR constants & 2.302 & 2.417 & 0.198 \\
PINN-refit ODE & 4.008 & 2.850 & 0.261 \\
Classical NLS ODE & 0.987 & 1.823 & 0.378 \\
\bottomrule
\end{tabular}
\end{table}

In the seed-47 example, the selected PySR expression is a linear saturation law.
Both refitters increase its forecast error. The direct PINN-network forecast
has smaller error than the solved PINN-refit ODE but larger error than the
original PySR ODE. The neural trajectory and the trajectory generated by its
fitted constants are therefore evaluated separately.
Both examples retain all 20 PINN starts. Their law-error ranges are
\([0.261,0.321]\) at seed 45 and \([0.434,0.464]\) at seed 47, entirely
above the respective original PySR errors of 0.198 and 0.427. Thus the
law-error deterioration in these two examples is not confined to the selected
start. Appendix~\ref{app:seeds} gives the mean, standard deviation, median,
interquartile range, and extrema across starts.

\begin{table}[htbp]
\centering
\caption{Seed-47 summary for the linear candidate
\(f(V)=3.288-0.1394V\) from
seed 47, \(N=6\), \(\sigma=0.20\), and
\(\lambda_{\mathrm{phys}}=0.1\). The PINN representative uses initialization
seed 55, selected from 20 starts by training MSE. The PINN network output and the
solved PINN-refit ODE are distinct objects.}
\label{tab:seed47}
\begin{tabular}{lccc}
\toprule
Method & Training MSE & Forecast RMSE & \(\Ef\) \\
\midrule
PySR ODE & 13.284 & 2.199 & 0.427 \\
PINN-refit ODE & 13.473 & 4.245 & 0.434 \\
PINN network output & 12.802 & 2.690 & -- \\
Classical NLS ODE & 12.781 & 3.532 & 0.400 \\
\bottomrule
\end{tabular}
\end{table}

Status-gated comparisons omit rows marked as domain failures or unstable refits
from accuracy summaries and count them in method-specific failure rates, even
when an intermediate numerical forecast appears favorable. The single-candidate
examples illustrate metric separation but do not support method-level inference.

Across the six sparse noise conditions at data seed 47 and
\(\lambda_{\mathrm{phys}}=0.1\), the initial filter retains 52 of 64
generated expressions. PINN and classical refitting retain 40 and 44
expressions, respectively. Relative to the original constants of each retained
structure, PINN refitting increases \(\Ef\) in 35 of 40 comparisons, whereas
classical refitting decreases it in 36 of 44. These are descriptive,
within-structure comparisons across the retained sets, not independent
replicate estimates or a comparison on identical retention sets.

\subsection{Refit success and physics-weight sensitivity}
The refitters have distinct potential failure modes. An unchanged classical
parameter vector does not by itself identify the cause of termination, whereas a PINN loss can
become nonfinite when a candidate is evaluated outside its supported domain.
Candidate-aware maps are designed to constrain supported domains; unsupported
expressions and residual training instability remain explicit outcomes.
The implementation also records unsupported parameter normalization when its
one-to-one scaling of constants cannot represent a product of learned
constants outside logarithms. This status occurs before optimization and is
distinct from a domain violation or a refitted equilibrium that is unreachable
from the initial condition.

Each candidate is classified into four cases:
\begin{equation}
\begin{array}{c|c}
\text{Case} & \text{Meaning}\\
\hline
\text{Both succeed} & \text{classical and PINN refits both pass}\\
\text{Classical only} & \text{classical refit passes, PINN fails}\\
\text{PINN only} & \text{PINN refit passes, classical refit fails}\\
\text{Neither} & \text{both refits fail}
\end{array}
\end{equation}
Numerical and domain failures are recorded separately for each method.
Within the classical-success category, parameter movement describes whether
the returned constants changed.
Within the PINN-failure category, unsupported candidate domains are separated
from training instability.

Physics-weight sensitivity is summarized on the grid
\begin{equation}
    \lambda_{\mathrm{phys}}\in\{0.1,1,10\}.
\end{equation}
Table~\ref{tab:physics-weight} pools the selected PINN representative from each
retained candidate across four noise levels at data seed 47. Each representative
is selected by training MSE from the retained PINN starts. The table reports
medians over these representatives, rather than over individual starts. The
smallest physics weight gives the smallest median training, forecast, and
law-level errors among the retained candidate sets. Increasing the physics
weight does not monotonically improve forecasting. Shared data realizations,
different retention rates, and selection of starts limit this
comparison to a description of the retained candidate sets.
At each physics weight, the 36 initially accepted expressions generate 720
configured PINN starts. The retained-start counts are 512, 559, and 537 at
weights 0.1, 1, and 10. Each weight has 120 starts with unsupported parameter
normalization; the remaining failures are domain violations (68, 22, and 63)
or refitted equilibria outside the initial condition's basin (20, 19, and 0).

\begin{table}[htbp]
\centering
\caption{Medians over selected PINN representatives for
\(\sigma=5\%,10\%,15\%,20\%\), \(N=6\), and data seed 47.
Each expression contributes its retained start with the smallest training MSE.
The retained candidate counts are 26/44, 29/44, and 28/44 for physics weights
0.1, 1, and 10, respectively; 36 of the 44 original expressions pass the
initial filter at each weight. Candidate expressions and repeated starts are not
independent data realizations.}
\label{tab:physics-weight}
\begin{tabular}{lccc}
\toprule
\(\lambda_{\mathrm{phys}}\) & Training-trajectory MSE & Forecast RMSE & \(\Ef\)\\
\midrule
0.1 & 3.8852 & 1.5939 & 0.3910 \\
1 & 6.9876 & 3.2325 & 0.4545 \\
10 & 9.9859 & 2.3051 & 0.4709 \\
\bottomrule
\end{tabular}
\end{table}

\subsection{Sampling density, spline degree, and a fixed-family baseline}
Table~\ref{tab:misspec} compares a directly fitted Gompertz model with
PINN-refitted symbolic candidates for one Gompertz data realization per regime
(data seed 47). The observations span \([0,15]\), and forecast errors use \(N\)
noiseless times on \([15.5,25.5]\). Candidate counts refer to expressions in a
PySR run, not repeated data trials. The fixed-Gompertz column contains the error
of one least-squares fit; the candidate column is the median error among
expressions that return a finite forecast.

This diagnostic uses a single PINN start (seed 42), 500 epochs, and candidate
parameter learning rate \(10^{-5}\), rather than the main refitting
configuration. A finite forecast is its inclusion criterion and does not
establish optimizer convergence or admissibility of the refitted expression.
With a cubic smoothing spline, the median candidate forecast error exceeds the
fixed-Gompertz error in all three sampling regimes. In the dense regime with a
quintic spline, four expressions return finite forecasts, all with smaller
errors than the fixed-family baseline. Only one of these four expressions has
Gompertz form. These comparisons describe sensitivity to sampling and spline
construction within the given realizations.

Figure~\ref{fig:misspec} is a separate misspecification illustration. Its
trajectory is generated by a Richards law with \(\nu=2\) and fitted with a
Gompertz model. It is not the data-generating process for
Table~\ref{tab:misspec}.

\begin{figure}[htbp]
\centering
\includegraphics[width=0.70\textwidth]{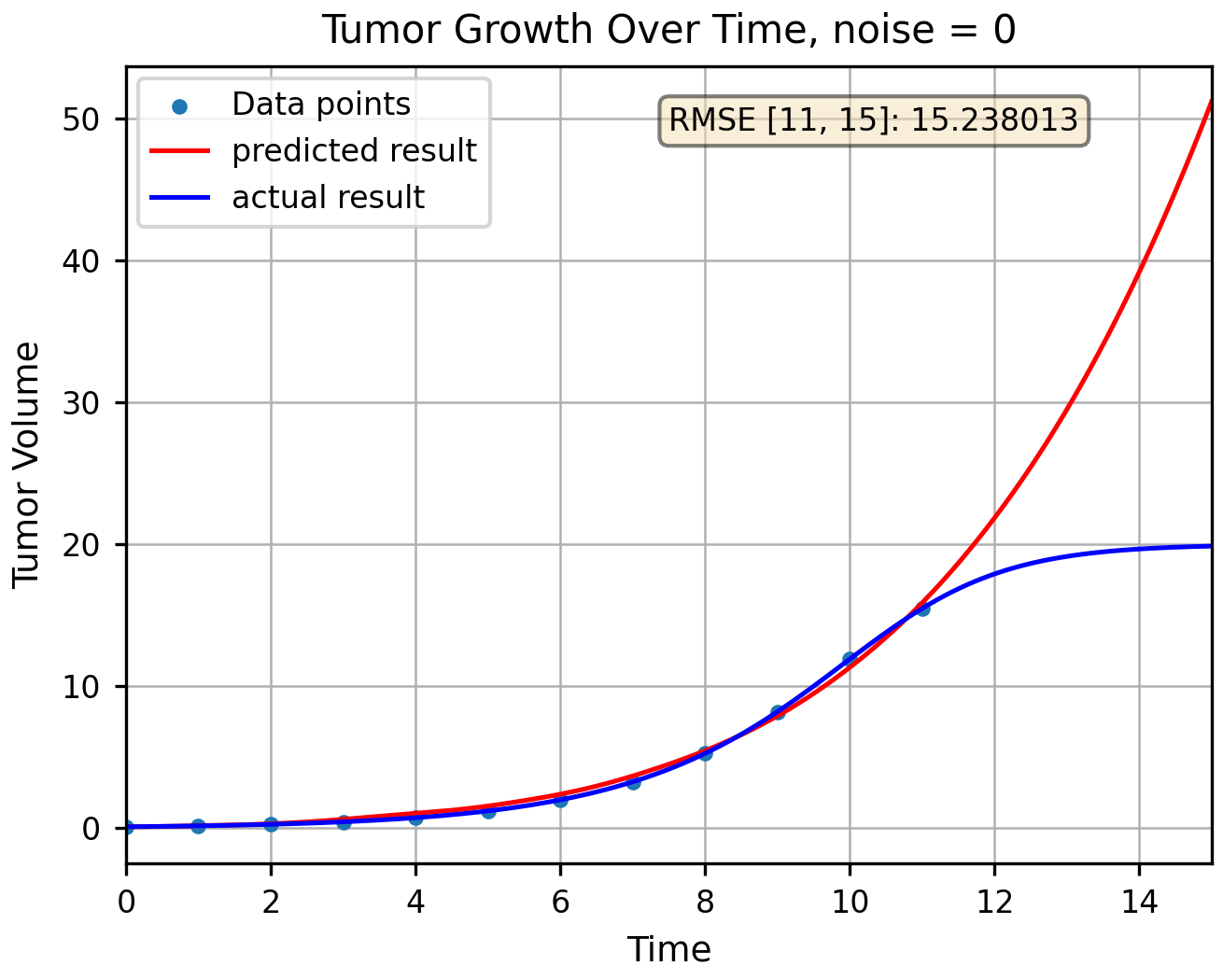}
\caption{A noiseless Richards example with \(\nu=2\), \(r=0.5\), \(K=20\),
\(V_0=0.1\), and 12 observations on \([0,11]\), fitted with a Gompertz family.
The forecast RMSE of 15.238 uses 200 times on \([11,15]\). This illustration
is distinct from the Gompertz-generated comparisons in Table~\ref{tab:misspec}.}
\label{fig:misspec}
\end{figure}

\begin{table}[htbp]
\centering
\caption{Gompertz-generated comparisons at data seed 47. The fixed-family error
is from one fit; the candidate error is a median over finite PINN-refitted
forecasts. ``Lower error'' counts candidate expressions that outperform the
fixed-Gompertz fit. All rows except the last use spline degree \(k=3\).}
\label{tab:misspec}
\small
\begin{tabular}{lcccccc}
\toprule
Regime & \(N\) & \(\sigma\) & Finite/total & \multicolumn{2}{c}{Forecast RMSE} & Lower error \\
\cmidrule(lr){5-6}
& & & & Fixed Gompertz & Candidate median & \\
\midrule
Sparse & 12 & 0.05 & 7/11 & 0.214 & 0.325 & 0/7 \\
Medium & 20 & 0.03 & 8/11 & 0.146 & 0.308 & 2/8 \\
Dense, \(k=3\) & 40 & 0.01 & 8/11 & 0.112 & 0.204 & 1/8 \\
Dense, \(k=5\) & 40 & 0.01 & 4/12 & 0.112 & 0.078 & 4/4 \\
\bottomrule
\end{tabular}
\end{table}

\section{LLM ranking of archived candidate sets}
\label{sec:llm-evaluation}
\subsection{Records and scoring}
The case study comprises six ranking reports dated 20--21 September 2026,
one for each noise level
\(\sigma\in\{0,0.05,0.10,0.15,0.20,0.30\}\) at data seed 47.
Each pool combines original PySR, PINN-refitted, and classical-refitted
expressions from the historical \(\lambda_{\mathrm{phys}}=0.1\) runs.
The pool sizes are 24, 26, 27, 20, 27, and 21, totaling 145 entries.
Every response ranks all supplied IDs exactly once, with its stated top choice
agreeing with rank one. The common prompt is reproduced in
Appendix~\ref{app:prompt}; the original prompt, input files, and unedited ranking
PDFs accompany the reproducibility archive. The records identify neither the
model version nor decoding settings, so the unit of analysis is the recorded
ranking rather than a specified model's repeatable performance distribution.

All 145 visible expressions and training errors match their numerical source
records. Forecast and law errors are recomputed from these exact expressions,
using an independent ODE integration and quadrature calculation, with all
metrics finite. Forecast discrepancies from the source records are below
\(2\times10^{-9}\), and law-error discrepancies are below \(4\times10^{-14}\).
Algebraic classification tests whether
\(f(V)/V=A-B\log V\) with constant \(A\) and \(B>0\); none of the 145
entries belongs to the Gompertz family. Conditional
\(\mathrm{LLMTop1}\) family-recovery accuracy is therefore undefined for all
six pools.

\subsection{Top-choice performance}
Table~\ref{tab:llm-outcomes} compares the recorded top choices with the
minimum-training-MSE baseline. All six LLM top choices correspond to
classical-refitted expressions, although that provenance was omitted from the
visible table. This observation describes the supplied mixed pools and does
not isolate a preference for a refitting method.

\begin{table}[htbp]
\centering
\caption{Retrospective scoring of the six archived LLM rankings. \(n\) is the
pool size; ``Train'' selects the smallest supplied training MSE. IDs are local
to each pool. Both forecast RMSE and law error are computed on the exact
historical expressions, not the later numerical refits.}
\label{tab:llm-outcomes}
\small
\begin{tabular}{ccccrrrr}
\toprule
\(\sigma\) & \(n\) & \multicolumn{2}{c}{Selected ID} &
\multicolumn{2}{c}{Forecast RMSE} & \multicolumn{2}{c}{\(\Ef\)} \\
\cmidrule(lr){3-4}\cmidrule(lr){5-6}\cmidrule(lr){7-8}
& & LLM & Train & LLM & Train & LLM & Train \\
\midrule
0.00 & 24 & C6  & C21 & 0.139 & 0.123 & 0.330 & 0.123 \\
0.05 & 26 & C9  & C17 & 0.289 & 0.526 & 0.177 & 0.091 \\
0.10 & 27 & C12 & C12 & 0.815 & 0.815 & 0.235 & 0.235 \\
0.15 & 20 & C9  & C18 & 2.125 & 1.296 & 0.346 & 0.329 \\
0.20 & 27 & C15 & C24 & 3.532 & 1.549 & 0.400 & 0.355 \\
0.30 & 21 & C6  & C6  & 4.206 & 4.206 & 0.397 & 0.397 \\
\bottomrule
\end{tabular}
\end{table}

Relative to the training-error baseline, the LLM choice has lower forecast
error in one pool, equal error in two, and higher error in three. Its law error
is equal in two pools and higher in four. Equality here uses absolute tolerance
\(10^{-8}\); in both equal cases the selected IDs are identical. Forecast
regret relative to the best candidate in each complete pool is positive in
all six cases, ranging from 0.027 to 1.983. Law-error regret ranges from zero
to 0.207 and vanishes only at \(\sigma=0.10\). Thus the archived rankings do
not demonstrate added predictive or law-recovery value over minimum training
error. The six noise conditions share a data seed and candidate-generation
pipeline; these counts are descriptive, not independent-replicate significance
tests.

\subsection{Scientific explanations}
The responses often favor smooth saturation behavior and penalize restricted
logarithmic domains, consistent with the prompt's robustness-first ordering.
Their explanations nevertheless require mathematical verification. For
example, the 20\% noise response treats logarithmic singularities at particular
volumes as lying within the forecast interval \([15.5,25.5]\), which is an
interval of \emph{time}. Comparing a state-domain boundary with a time horizon
does not establish whether a trajectory approaches that boundary. The 5\%
response describes C14 as having negative low-volume growth, although its
expression is a positive quadratic factor multiplied by \(20.1231-V\),
and hence is positive below that equilibrium. These errors illustrate the need
to verify ranking explanations by domain analysis and numerical evaluation.

\section{Discussion}
The numerical comparisons separate three sources of uncertainty: the equation
family proposed by symbolic search, the constants obtained by refitting, and
the ranking of accepted candidates. These uncertainties affect different
aspects of extrapolation.

First, sparse tumor-growth law discovery depends on both identifiability and
optimization. When observations are concentrated near saturation, several
laws produce similar trajectories but different vector fields. Derivative-based
phase-space symbolic regression can rank plausible saturation
surrogates above the generating law.

Second, trajectory fit is not law recovery. The log-saturation family
\(\log(C-V)\) illustrates the point: its equilibrium can be moved close to the
true carrying capacity, producing good forecasts in some cases, while the family
still has a parameter-dependent restricted domain that may exclude the
evaluation interval, the wrong near-zero behavior, and the wrong vector-field
shape. Law-level error \(\Ef\) therefore distinguishes a curve fit from recovery
of the governing equation.

Third, PINN refitting is best viewed as a regularized refit of a proposed law,
not as a mechanism that rescues misspecified law families. It can improve
forecast behavior for some accepted candidates, and it can fail on others when
a candidate domain is unsupported or when optimization becomes unstable. These
are mathematically different events from a
classical no-update solution, whose interpretation also requires the optimizer's
termination status. Candidate-aware state maps constrain the region in which
the physics residual is evaluated. Numerical errors must be read together with
the integration and optimization outcomes, especially when a summary retains
only a subset of candidate expressions or initializations.

A fixed large loss assigned after a NaN does not solve the domain problem, and
ad hoc parameter perturbations change the candidate being refit. Domain
constraints must enter before the symbolic right-hand side is evaluated. For
simple interval domains this can be achieved by smooth output
parameterizations; candidates with unsupported domains are recorded as such.

In the six archived cases, the LLM choices improve on the training-error
baseline in only one forecast comparison and no law-error comparison.
The volume--time confusion in one explanation also shows that a ranking
rationale may be mathematically incorrect. Deterministic checks and
solved-ODE evaluation remain necessary for LLM-selected equations.

The benchmark is synthetic,
one-dimensional, autonomous, and focused on monotone untreated growth. Forecast
and vector-field errors are available because the generating law is known; they
are not directly observable in clinical data. The symbolic stage depends on
spline differentiation, and its expressions are nested within a common data
realization rather than independent replicates. PINN conclusions depend on
initialization and optimization settings. The LLM results concern six saved
responses to one prompt, with mixed-arm, nonrandomized pools and
future-informed historical eligibility. They do not establish causal ranking
benefits, robustness to presentation order, or performance across LLMs.

\section{Conclusion}
In this synthetic benchmark, accurate trajectory fits and forecasts can
coexist with incorrect growth laws. Among retained candidates under the tested
settings, PINN refitting often increases vector-field error, whereas classical
nonlinear least squares usually reduces it. The archived LLM rankings show no
consistent improvement over selection by training error and include
mathematically incorrect explanations. The neural surrogate, the solved
candidate ODE, and the recovered vector field therefore require separate
evaluation. Neither low training error nor a persuasive ranking rationale
certifies recovery of the governing law.

\section*{Acknowledgments}
W. Wang was partially supported by the Simons Foundation Travel Support for
Mathematicians (No.~0007730). The authors acknowledge use of AI tools.

\section*{Data and code availability}
The supplementary reproducibility archive contains the symbolic candidate
tables, refitting code and numerical records, the archived LLM prompt and input
files, the six original ranking reports, and the retrospective scoring code. \url{https://github.com/Victor230791/AI_Assisted_Tumor_Growth_Law_Discovery_Code_Base}

\appendix
\section{Archived LLM ranking prompt}
\label{app:prompt}
The following text reproduces the substantive instructions in the archived
prompt, with mathematical notation typeset. The original PDF also requests
export of the response as a dated PDF named after the input CSV. Its reference
to a randomized table is part of the original prompt; the stored inputs use
sequential IDs and interleaved method order, as described in Section~3.5.

\begin{quote}\small
You are an independent, blinded evaluator in a symbolic tumor-growth ODE
model-selection experiment. Treat this as an isolated task. Do not use previous
conversations, memory, web browsing, external sources, established laws, or
results from any other experiment.

The tumor volume \(V(t)>0\) is modeled by the autonomous ODE
\[
    \frac{dV}{dt}=f(V).
\]
The data-generating equation, its parameters, all future observations, forecast
errors, and law-level errors are intentionally withheld.

Every candidate below has already passed the following deterministic
requirements:
\begin{enumerate}[leftmargin=2em,itemsep=1pt]
    \item the right-hand side is real and finite on the required positive-volume
    interval;
    \item the trajectory of their solutions with the initial condition
    \(V(0)=0.1\) remains positive and numerically solvable;
    \item the candidate has a finite positive stable equilibrium.
\end{enumerate}

Rank the accepted candidates from most to least scientifically defensible for
extrapolating tumor growth. Apply these criteria in order:
\begin{enumerate}[leftmargin=2em,itemsep=1pt]
    \item mathematical and biological robustness, including plausible positive
    growth below saturation, stable saturation, and no evident fragility from a
    nearby singularity or restricted domain;
    \item agreement with the observed data, using only the supplied
    training-trajectory MSE;
    \item parsimony and interpretability, using the supplied complexity; and
    \item lower complexity when the preceding considerations are essentially
    tied.
\end{enumerate}
Do not favor a candidate merely because it resembles a proposed or known growth
law. Do not identify the hidden generating equation, refit coefficients,
simulate trajectories, calculate missing quantities, modify or simplify an
expression, reject an accepted candidate, or propose a new equation. Do not
infer anything from candidate IDs or row order.

Rank every candidate exactly once. Return a Markdown table with columns
\texttt{Rank}, \texttt{Candidate ID}, \texttt{Main strength},
\texttt{Main concern}, and \texttt{Reason for rank}; then identify one top
choice by candidate ID with a two-sentence explanation; then give a two- or
three-sentence comparison of the candidate set. Do not output JSON, code, a
response schema, or a rewritten prompt.

Randomized accepted candidate table is attached in the CSV file. Refer to the
TXT file for trial information.
\end{quote}

\section{Candidate-table schema}
The archived ranking inputs and retrospective evaluation records have the
following fields. Hidden columns were absent from the supplied CSV and TXT
files; historical pool eligibility was nevertheless future-informed.
\begin{center}\small
\begin{tabularx}{\textwidth}{@{}>{\raggedright\arraybackslash}p{0.26\textwidth}>{\raggedright\arraybackslash}p{0.16\textwidth}>{\raggedright\arraybackslash}X@{}}
\toprule
Field & LLM visibility & Meaning \\
\midrule
Trial information & Visible & noise level and training/forecast intervals \\
Training observations & Visible & only the observed training \((t,V)\) values \\
Candidate ID and expression & Visible & sequential ID and fitted right-hand side \\
Complexity & Visible & symbolic complexity score \\
Training-trajectory MSE & Visible & solved-ODE fit to the observed sparse points \\
Acceptance statement & Visible & shared statement that every listed candidate passed the required filters \\
Future volumes and forecast RMSE & Hidden & author-computed evaluation of the archived selections \\
\(\Ef\) and generating-family match & Hidden & synthetic-oracle evaluation fields \\
LLM rank and rationale & Output & recorded ranking and scientific explanation \\
\bottomrule
\end{tabularx}
\end{center}

\section{Variation across PINN initializations}
\label{app:seeds}
Table~\ref{tab:seed-statistics} summarizes the same two candidates as
Tables~\ref{tab:seed45} and~\ref{tab:seed47}, across initialization seeds
43--62. All 20 starts pass for each example. For a metric \(M_s\), SD uses the
population convention
\[
\mathrm{SD}=\left(\frac1{20}\sum_{s=1}^{20}(M_s-\overline M)^2\right)^{1/2}.
\]
IQR is the difference
between the empirical 75th and 25th percentiles. These are within-candidate
initialization summaries, not confidence intervals over independent data sets.
The numerical archive also retains every start's status, restored expression,
network metrics, and fitted-ODE metrics.

\begin{table}[htbp]
\centering
\caption{Across-start PINN-refit ODE errors for the two \(20\%\)-noise examples.
Each row contains 20 valid starts; values are rounded to three decimals.}
\label{tab:seed-statistics}
\small
\begin{tabular}{llrrrrrr}
\toprule
Data seed & Metric & Mean & SD & Median & IQR & Min & Max \\
\midrule
45 & Training MSE & 5.745 & 0.919 & 5.864 & 0.964 & 4.008 & 7.740 \\
45 & Forecast RMSE & 3.111 & 0.096 & 3.140 & 0.085 & 2.850 & 3.279 \\
45 & \(\Ef\) & 0.293 & 0.015 & 0.296 & 0.014 & 0.261 & 0.321 \\
\addlinespace
47 & Training MSE & 13.811 & 0.373 & 13.715 & 0.284 & 13.473 & 14.834 \\
47 & Forecast RMSE & 4.447 & 0.223 & 4.419 & 0.206 & 4.064 & 5.047 \\
47 & \(\Ef\) & 0.443 & 0.008 & 0.441 & 0.007 & 0.434 & 0.464 \\
\bottomrule
\end{tabular}
\end{table}

\end{document}